\documentclass[twoside]{article}

\usepackage[preprint]{aistats2026}
\usepackage[round]{natbib}

\usepackage{enumitem}
\usepackage{mathtools}
\usepackage{algorithmic}
\usepackage{amsmath,amssymb,amsthm}
\usepackage{hyperref}
\usepackage[capitalize,noabbrev]{cleveref}
\usepackage{url}
\usepackage{booktabs}

\newtheorem{theorem}{Theorem}
\newtheorem{proposition}{Proposition}
\newtheorem{corollary}{Corollary}

\theoremstyle{definition}
\newtheorem{definition}{Definition}
\newtheorem{assumption}{Assumption}
\theoremstyle{remark}
\newtheorem{remark}{Remark}
 
\newcommand{\E}{\mathbb{E}}
\newcommand{\R}{\mathbb{R}}
\newcommand{\Prob}{\mathbb{P}}
\newcommand{\EU}{\mathrm{EU}}
\newcommand{\AU}{\mathrm{AU}}
\newcommand{\TU}{\mathrm{TU}}
\newcommand{\cA}{\mathcal{A}}
\newcommand{\cX}{\mathcal{X}}
\newcommand{\cQ}{\mathcal{Q}}
\newcommand{\dTV}{d_{\mathrm{TV}}}
\newcommand{\Bern}{\mathrm{Bern}}

\begin{document}

%

%

\twocolumn[

\aistatstitle{Epistemic Uncertainty Is Not the Reducible Kind}

\aistatsauthor{ Robin Young }

\aistatsaddress{ University of Cambridge } ]

\begin{abstract}
The standard taxonomy of predictive uncertainty defines epistemic uncertainty as the part removable by collecting more data, while the standard measure identifies it with a mutual-information term. We prove the definition and the measure are extensionally inconsistent. On an explicit construction, the measure assigns all uncertainty to the epistemic class, yet no quantity of training data reduces it. Reducibility is instead a property of the pair (uncertainty, acquisition class), and the dichotomy resolves into three parts: aleatoric, sample-reducible epistemic, and mechanism-reducible epistemic uncertainty. An exact identity for the value of an observation shows that in-distribution data never reduces mechanism-irreducible uncertainty and generically increases it. Ensemble disagreement, the deployed epistemic estimate, tracks the training procedure rather than the epistemic term. It collapses to zero beneath a positive truth under consistent training, and equals hyperparameter-scaled initialization noise under interpolation. A finite-sample falsification test and seed-swept experiments confirm the theory.
\end{abstract}

\section{Introduction}
\label{sec:intro}
 
The most-cited definitional sentence in deep-learning uncertainty quantification states that epistemic uncertainty is ``uncertainty which can be explained away given enough data'' \citep{kendall2017uncertainties}. The companion survey of \citet{hullermeier2021aleatoric} organizes the field around the same criterion: aleatoric uncertainty is the irreducible part, epistemic uncertainty the part that can in principle be reduced. Downstream work uses reducibility not as a property but as a definition, writing ``epistemic (reducible)'' as a bare parenthetical identity \citep[e.g.][]{henning2022ood,chen2026querylevel}, and the commitment is operational. Active-learning systems route high-epistemic inputs to additional sampling \citep{galislam2017deep,walmsley2020galaxy}, method papers validate estimators by exhibiting epistemic uncertainty shrinking with dataset size \citep{kendall2017uncertainties}, and at least one widely deployed estimator hardcodes the decay into its functional form \citep{amini2020deep}.
 
This paper demonstrates that the commitment is false constructively. Consider covariates $\cX = \{0,1\}$ with all training data drawn at $X = 0$, a fair latent bit $J$, labels at $X = 0$ that are fair coin flips regardless of $J$, and labels at $X = 1$ equal to $J$. Query the label at $X = 1$. The training data carries no information about $J$ at any sample size, so the mutual information between $J$ and the prediction target, which is the field's own epistemic measure, equals $\log 2$ forever yet a single observation at $X = 1$ reveals $J$ exactly. The uncertainty is epistemic by every substantive criterion (an oracle who knows $J$ has none of it), irreducible by any quantity of deployment-distribution data, and removable by one sample from elsewhere. The current taxonomy has no box for it.
 
The mechanism is that reducibility is not a property of uncertainty but of the pair (uncertainty, acquisition class). The textbook definition implicitly quantifies over one acquisition class (more i.i.d.\ draws from the training distribution) and everything that class cannot buy is thus misfiled. Once the quantifier is made explicit, the dichotomy resolves into a trichotomy of aleatoric uncertainty, sample-reducible epistemic uncertainty, and mechanism-reducible epistemic uncertainty, the last removable only by changing where the data comes from.
 
\textbf{Contributions.}
\begin{enumerate}
\item \textbf{A non-partition theorem and an inconsistency theorem} (\Cref{sec:partition}). We formalize acquisition classes, exhibit the construction above, and prove that no classification of uncertainty can simultaneously satisfy the field's measure (epistemic $=$ the mutual-information term) and the field's definition (epistemic $=$ reducible by more training data). The two commitments disagree extensionally on a constructible instance, and either retreat concedes the other (\Cref{thm:inconsistency}). A quantitative version gives the exchange rate between acquisition mechanisms: deployment-distribution data has rate exactly zero while off-support probes buy geometric decay at the Bhattacharyya exponent of the component laws (\Cref{thm:rates}).
\item \textbf{The exact value of an observation} (\Cref{sec:value}). A chain-rule identity gives the epistemic-uncertainty reduction from any observation as an information gain minus a redundancy term (\Cref{thm:value}). Two corollaries follow. Under exact unidentifiability, in-support data never reduces the epistemic term and generically increases it. The optimal probe location maximizes information about the latent minus its redundancy given the target, a term absent from information-gain acquisition rules \citep{houlsby2011bald}.
\item \textbf{Ensemble disagreement tracks optimization, not the epistemic term} (\Cref{sec:ensembles}). In the amortized regime, every dispersion functional converges to zero while the epistemic term is fixed and positive (\Cref{thm:collapse}); in the complementary interpolation regime, ensemble variance equals $\alpha^2 \lVert P_N x_\star \rVert^2$ in closed form, which is initialization noise at a hyperparameter-set scale, unbounded in both directions relative to the truth (\Cref{thm:inflation}).
\item \textbf{Experiments} (\Cref{sec:test,sec:experiments}). Monotonicity of the epistemic term under exact unidentifiability yields a level-$\alpha$ finite-sample test that rejects mechanism-irreducibility upon a significant decrease. Seed-swept experiments confirms the strict increase ($p \le 4 \times 10^{-10}$ across all increments), rejection rates of $0/10$ and $10/10$ on the two controls, closed-form agreement of the inflation regime to median relative error $1.9\%$, and a three-order-of-magnitude collapse of the ensemble estimator against a flat truth.
\end{enumerate}

We do not claim the conceptual observation that the dichotomy conflates the source of uncertainty with its reducibility. \citet{baan2023uncertainty} state it in survey form, \citet{derkiureghian2009aleatory} argue model-relativity, and the survey of \citet{hullermeier2021aleatoric} itself asks what reducible means before adopting the criterion. Our contribution is the proof with the formal counterexample, the extensional inconsistency, the exact value identity with its negative-value corollary, the quantitative rates, and the two-sided ensemble separation.
 
\section{Related Work}
\label{sec:related}
 
\textbf{The taxonomy and its critics.} The aleatoric/epistemic distinction enters machine learning through \citet{derkiureghian2009aleatory}, \citet{kendall2017uncertainties}, and the survey of \citet{hullermeier2021aleatoric}, with reducibility as the organizing criterion throughout. A critical literature has begun to form. \citet{baan2023uncertainty} observe that source and reducibility are orthogonal axes, \citet{gruber2023sources} note that reducibility depends on what counts as knowledge and an ICLR blog post works through reducibility paradoxes informally \citep{kirchhof2025reexamining}. Closest to us, \citet{ruegamer2026epistemic} study epistemic uncertainty under parameter non-identifiability in overparametrized networks, showing persistent parameter uncertainty when the function is fully identified. Their setting is a dual of ours as their irreducible uncertainty lives in weight space and does not touch the predictive. On the other hand ours lives in the predictive target $I(J; Y_\star \mid D_n)$, which is the quantity practitioners act on. None of these works state the non-partition counterexample or the measure--definition inconsistency as theorems, which is the gap we fill.
 
\textbf{Measures and their axiomatics.} The entropy decomposition with the mutual-information term as the epistemic measure is standard \citep{depeweg2018decomposition,smith2018understanding,malinin2018predictive}. \citet{wimmer2023quantifying} critique the measure axiomatically, asking whether conditional entropy and mutual information are internally well-behaved. Our \Cref{thm:inconsistency} is orthogonal as we show the measure and the reducibility definition are jointly inconsistent regardless of either's internal axiomatics. \citet{bengs2022pitfalls} prove that loss minimization does not incentivize faithful second-order uncertainty. Our \Cref{thm:collapse} is complementary, showing that even correctly incentivized first-order training destroys the ensemble's epistemic signal.
 
\textbf{Acquisition.} BALD \citep{houlsby2011bald,kirsch2019batchbald} acquires the point maximizing $I(y; \theta \mid x, D)$, the mutual information between observation and model parameters. Prediction-oriented variants \citep{bickfordsmith2023prediction} weight information by downstream relevance. Our \Cref{cor:acquisition} supplies the exact latent-information form of the correction and a phenomenon outside both frameworks with observations whose net value is strictly negative.
 
\textbf{Floors from indistinguishability.} The quantity our construction holds positive is the log-loss minimum excess risk of \citet{xu2022minimum} in its non-vanishing regime. We regard it as the instance of one phenomenon, floors from indistinguishability, and return to this in the discussion. \citet{foster1998asymptotic} supply the deep precedent that the field's trust certificates certify less than assumed, that calibration is achievable against an adversarial sequence. \citet{osband2018randomized} establish the sample-then-optimize mechanism underlying our \Cref{thm:inflation}, which we use contrapositively.
 
\section{Setup}
\label{sec:setup}
 
\subsection{Predictive problems and the standard decomposition}
 
\begin{definition}[Predictive problem]
\label{def:problem}
A predictive problem is a tuple $P = (\Pi, \rho, \{p(\cdot \mid x, j)\}, x_\star)$: a prior $\Pi$ over a latent $J$, a sampling distribution $\rho$ on covariates $\cX$, conditional label laws, and a query point. Data $D_n = \{(X_i, Y_i)\}_{i=1}^n$ with $X_i \overset{\text{iid}}{\sim} \rho$ and $Y_i \sim p(\cdot \mid X_i, J)$ conditionally independent given $J$; the target is $Y_\star \sim p(\cdot \mid x_\star, J)$, conditionally independent of $D_n$ given $J$.
\end{definition}
 
The standard decomposition of total predictive uncertainty \citep{depeweg2018decomposition,hullermeier2021aleatoric} is
\begin{equation}
\label{eq:decomp}
\underbrace{\E\, H(Y_\star \mid D_n)}_{\TU(n)} \;=\; \underbrace{\E\, H(Y_\star \mid J, D_n)}_{\AU(n)} \;+\; \underbrace{\E\, I(J; Y_\star \mid D_n)}_{\EU(n)},
\end{equation}
and the mutual-information term $\EU(n)$ is the field's epistemic measure. The field's epistemic definition is reducibility, namely that epistemic uncertainty is the part that vanishes with more training data. To make the definition precise we must say which data.

\begin{definition}[Acquisition classes]
\label{def:acquisition}
An acquisition plan is a (possibly adaptive) policy for collecting additional observations. An acquisition class $\cA$ is a set of plans containing the null plan. The $\cA$-irreducible epistemic uncertainty is
\begin{equation}
\EU_\infty^{\cA} \;:=\; \inf_{a \in \cA,\; k \ge 0}\; \E\, I\bigl(J;\, Y_\star \mid D_n,\, Z^a_{1:k}\bigr),
\end{equation}
where $Z^a_{1:k}$ is the data acquired by plan $a$ with budget $k$. Since $\cA$ contains the null plan, $\EU_\infty^{\cA} \le \EU(n)$, and $\cA \subseteq \cA'$ implies $\EU_\infty^{\cA} \ge \EU_\infty^{\cA'}$. We write $\cA_\rho$ for the class consisting of unlimited additional i.i.d.\ draws from $\rho$.
\end{definition}
 
More formally, the textbook commitment of uncertainty is epistemic if and only if $\EU_\infty^{\cA_\rho} = 0$. Relative to a reference class this induces a trichotomy: \emph{aleatoric} uncertainty $\AU$, which survives conditioning on $J$; \emph{sample-reducible epistemic} uncertainty $\EU(n) - \EU_\infty^{\cA_\rho}$, removable by more of the same data; and \emph{mechanism-reducible epistemic} uncertainty $\EU_\infty^{\cA_\rho} - \EU_\infty^{\cA'}$ for richer classes $\cA'$, removable only by changing where the data comes from. The results below show the third class is nonempty, that it is exactly where the field's commitments collide, and that it is invisible to the field's estimators.
 
\subsection{Constructions}
 
\begin{definition}[Two-region Bernoulli]
\label{def:construction}
Let $\cX = \{0,1\}$, $\rho = \delta_0$, and $J \sim \Bern(\tfrac12)$. Labels: $Y \mid X{=}0, J \sim \Bern(\tfrac12)$ for both values of $J$; and $Y \mid X{=}1, J \sim \Bern(q_J)$ with $q_0 = \varepsilon$, $q_1 = 1 - \varepsilon$, $\varepsilon \in [0, \tfrac12)$. The query is $x_\star = 1$. We call $\varepsilon = 0$ the noiseless construction, in which $Y_\star = J$ almost surely.
\end{definition}

For intuition, let us walk through the noiseless case. Every training point sits at $X = 0$ and carries a fair coin flip uncorrelated with $J$, so a dataset of any size is exactly as informative about $J$ as an empty one. The target $Y_\star$ is $J$. So the predictive uncertainty at the query is one bit, all of it is mutual information with the latent, none of it is aleatoric (conditioning on $J$ removes it entirely), and no quantity of training data touches it, yet a single label observed at $X = 1$ removes all of it. Every theorem following this can be read off this example as the noisy version ($\varepsilon > 0$) and the Gaussian-process instance below add rates and geometry, not phenomena.

The construction requires nothing but a latent that controls labels only off the sampling support. For the continuous instance we use kernels $k_j(x,x') = a_j(x) a_j(x') c(x,x')$ with amplitudes $a_j \equiv 1$ on $\mathrm{supp}\,\rho$ and $a_0(x_\star) \neq a_1(x_\star)$, so that the two components induce identical context laws while differing at the query. We say exact unidentifiability holds when the conditional law of any in-support observation given $(D_n, J)$ is the same for all $J$, and both constructions satisfy it.
 
\section{The Dichotomy Is Not a Partition}
\label{sec:partition}
 
\begin{theorem}[Non-partition]
\label{thm:partition}
In the noiseless two-region construction: (i) $\EU(n) = \log 2$ for all $n$, and $\EU_\infty^{\cA_\rho} = \log 2$; (ii) $\AU(n) = 0$ for all $n$; (iii) for any class $\cA'$ containing a single draw at $X = 1$, $\EU_\infty^{\cA'} = 0$.
\end{theorem}

\begin{proof}
The law of $D_n$ given $J = j$ is $(\delta_0 \otimes \Bern(\tfrac12))^{\otimes n}$ for both $j$, so $D_n \perp J$; since $Y_\star = J$, $\sigma(J, Y_\star) = \sigma(J)$ and hence $D_n \perp (J, Y_\star)$, and the same holds with any number of additional $\rho$-draws appended. Therefore $\EU(n) = I(J; Y_\star) = H(J) = \log 2$ under every plan in $\cA_\rho$, giving (i). $H(Y_\star \mid J, D_n) = H(J \mid J, D_n) = 0$, giving (ii). A draw $Z$ at $X = 1$ equals $J$ almost surely, so $I(J; Y_\star \mid D_n, Z) = 0$, giving (iii).
\end{proof}
 
\begin{corollary}[Trichotomy]
\label{cor:trichotomy}
There exists uncertainty that is (a) not aleatoric, (b) not reducible by any quantity of deployment-distribution data, and (c) removable by a single observation acquired through a different mechanism. Reducibility is therefore a property of the pair (uncertainty, acquisition class), and the dichotomy $\{\text{aleatoric}\} \sqcup \{\text{reducible epistemic}\}$ omits a nonempty class.
\end{corollary}

The noisy construction turns the trichotomy quantitative. Each acquisition mechanism comes with an exchange rate, and the rate for the deployment distribution is exactly zero.
 
\begin{theorem}[Exchange rates between mechanisms]
\label{thm:rates}
In the two-region construction with $\varepsilon \in (0, \tfrac12)$: (i) $\EU(n) = \log 2 - h(\varepsilon)$ for all $n$ under $\cA_\rho$, where $h$ is the binary entropy in nats; (ii) after $k$ i.i.d.\ probes at $X = 1$, writing $\rho_B := 2\sqrt{\varepsilon(1-\varepsilon)} \in (0,1)$ for the Bhattacharyya coefficient of the pair $(\Bern(\varepsilon), \Bern(1{-}\varepsilon))$,
\begin{equation}
\frac{\log 2 - h(\varepsilon)}{2}\, \rho_B^{\,2k} \;\le\; \EU_k \;\le\; \rho_B^{\,k}.
\end{equation}
\end{theorem}
 
\begin{proof}
(i) As in \Cref{thm:partition}, $D_n \perp (J, Y_\star)$, and $I(J; Y_\star) = H(Y_\star) - H(Y_\star \mid J) = \log 2 - h(\varepsilon)$ since the marginal of $Y_\star$ is $\Bern(\tfrac12)$ by symmetry of the prior. (ii) For the upper bound, let $\pi_j(z)$ be the posterior over $J$ given probes $z = z_{1:k}$. Data processing and $h(p) \le 2\sqrt{p(1-p)}$ give $I(J; Y_\star \mid Z{=}z) \le H(J \mid Z{=}z) \le 2\sqrt{\pi_0(z)\pi_1(z)}$. With $m$ the mixture law of $Z_{1:k}$ and $p_j$ the component laws, $\sqrt{\pi_0\pi_1} = \tfrac12 \sqrt{p_0 p_1}/m$, so $\E\sqrt{\pi_0\pi_1} = \tfrac12 \sum_z \sqrt{p_0(z)p_1(z)} = \tfrac12 \rho_B^k$ by tensorization of the Bhattacharyya coefficient; combine. For the lower bound, given $Z = z$ we have $I(J; Y_\star \mid z) = h(p(z)) - h(\varepsilon)$ with $p(z) = \varepsilon + \pi_1(z)(1-2\varepsilon)$; concavity of $h$ above its chord on $[\varepsilon, \tfrac12]$ gives $h(p) - h(\varepsilon) \ge 2(\log 2 - h(\varepsilon))\min_j \pi_j(z)$, and Cauchy--Schwarz applied to $\rho_B^k = \sum_z \sqrt{\min \cdot \max}$ with $\sum_z \max \le 2$ gives $\E[\min_j \pi_j] \ge \rho_B^{2k}/4$; combine.
\end{proof}
 
\begin{corollary}[Overlap exchange rate]
\label{cor:overlap}
Modify the two-region construction ($\varepsilon \in (0,\tfrac12)$) so that $X_i \overset{\text{iid}}{\sim} \rho_\eta := (1-\eta)\,\delta_0 + \eta\,\delta_1$, $\eta \in [0,1]$, and write $c_\varepsilon := \log 2 - h(\varepsilon)$, $\EU_\eta(n) := \E\, I(J; Y_\star \mid D_n)$. Then
\begin{equation}
\label{eq:overlap}
\tfrac{c_\varepsilon}{2}\bigl(1 - \eta\,(1-2\varepsilon)^2\bigr)^{n} \;\le\; \EU_\eta(n) \;\le\; \bigl(1 - \eta\,(1-\rho_B)\bigr)^{n},
\end{equation}
recovering \Cref{thm:rates}(ii) at $\eta = 1$ and the flat floor of \Cref{thm:partition} at $\eta = 0$. For every $\eta > 0$ the decay is geometric, so (D) holds verbatim off the boundary; but $\EU_\eta(n) \ge c_\varepsilon/4$ whenever $n\eta\,(1-2\varepsilon)^2 \le \tfrac12$, so the sample complexity of the deployment mechanism diverges as $1/\eta$ and a budget of $n$ samples buys essentially nothing unless $\eta \gtrsim 1/n$. Proof in \Cref{app:overlap}.
\end{corollary}

The factor-two exponent gap is the standard Bhattacharyya slack and tightening it changes nothing below. The substantive content is the contrast. The per-probe value of the off-support mechanism is a distinguishability coefficient of the component laws, the per-sample value of the deployment mechanism is identically zero, and \Cref{cor:overlap} interpolates. \Cref{thm:partition} is a boundary fact approached continuously, exact at $\eta = 0$ and operationally intact for finite budgets whenever the overlap is $O(1/n)$.

\begin{theorem}[Measure--definition inconsistency]
\label{thm:inconsistency}
No classification of predictive uncertainty simultaneously satisfies (M) epistemic uncertainty equals the mutual-information term of \cref{eq:decomp}, and (D) epistemic uncertainty is the part that vanishes with more training data.
\end{theorem}
 
\begin{proof}
In the noiseless construction, (M) assigns all uncertainty to the epistemic class: $\EU = \TU = \log 2$ and $\AU = 0$. (D) classifies the same uncertainty as non-epistemic, since $\EU_\infty^{\cA_\rho} = \log 2 \neq 0$; under the dichotomy it must then be aleatoric, contradicting $\AU = 0$.
\end{proof}
 
\begin{remark}
\label{rem:retreat}
The defense that reducible means reducible in principle by some data concedes (D) as stated and relativizes reducibility to an acquisition class, which is \Cref{cor:trichotomy}. With a fully unrestricted class, (D) collapses into (M). But every operational use of (D), such as acquisition routing, validation-by-decay, decay hardcoded into the estimator, quantifies over $\cA_\rho$, where \Cref{thm:partition} applies (\Cref{sec:discussion}). Either the constructed uncertainty is aleatoric, and the aleatoric box contains uncertainty that a single observation eliminates; or it is epistemic, and the epistemic box contains uncertainty that no quantity of training data reduces.
\end{remark}

\section{Exact Value of an Observation}
\label{sec:value}
 
\begin{theorem}[Value identity]
\label{thm:value}
For any additional observation $Z$, acquired through any channel,
\begin{align}
\label{eq:value}
\Delta\EU(Z) \;:=\; \E\bigl[I(J; Y_\star \mid D_n)\bigr] - \E\bigl[I(J; Y_\star \mid D_n, Z)\bigr] \\=\; \E\bigl[I(J; Z \mid D_n)\bigr] - \E\bigl[I(J; Z \mid Y_\star, D_n)\bigr]
\end{align}
\end{theorem}
 
\begin{proof}
Expand $I(J; Y_\star, Z \mid D_n)$ by the chain rule in both orders, rearrange, take expectations.
\end{proof}
 
The value of an observation is not its information about the latent, it is that information minus its redundancy given the target. Both consequences below are consequences of the redundancy term.

\begin{corollary}[In-support data never helps and strictly hurts]
\label{cor:insupport}
Under exact unidentifiability, for any in-support observation $Z$ the first term of \cref{eq:value} vanishes, so $\Delta\EU(Z) = -\E\, I(J; Z \mid Y_\star, D_n) \le 0$, and $\EU(n)$ is non-decreasing in $n$. On the Gaussian-process amplitude construction the inequality is strict for every $n$ and every parameterization. Writing $A$ for the (common) training Gram matrix, $b$ and $c_z$ for the unit-amplitude cross-covariance vectors of the query and probe, $\gamma = c_{\star\star} - b^\top A^{-1} b$, $\delta = c_{z\star} - c_z^\top A^{-1} b$, and $a_j = a_j(x_\star)$, the conditional law of the probe given $(D_n, Y_\star)$ under component $j$ is Gaussian with slope and variance
\begin{align}
\label{eq:schur}
g_j &= \frac{a_j\, \delta}{\sigma^2 + a_j^2 \gamma}, \\
v_j &= r - \frac{a_j^2\, \delta^2}{\sigma^2 + a_j^2 \gamma}, \\
r &= c_{zz} + \sigma^2 - c_z^\top A^{-1} c_z.
\end{align}
and $v_0 = v_1$ forces $a_0 = a_1$, with no exceptional parameter manifold. Hence $I(J; Z \mid Y_\star, D_n) > 0$ pointwise on the event $\{\delta \ne 0\}$, which has full measure for analytic strictly positive-definite correlations under non-atomic $\rho$. The increase is $O(\delta^2)$, inheriting the decay of $c(\cdot, x_\star)$.
\end{corollary}
 
The proof is the Schur-complement computation we defer to \Cref{app:schur} but the mechanism is visible in \cref{eq:schur}. Conditioning on a large $\lvert Y_\star \rvert$ makes the in-support probe informative about which amplitude produced it, because the probe--target cross-covariance scales with $a_j$. Each in-distribution sample therefore couples to the latent through the target and pushes the epistemic term up, not down.
 
\begin{corollary}[The acquisition objective]
\label{cor:acquisition}
The optimal probe location maximizes $I(J; Z_x \mid D_n) - I(J; Z_x \mid Y_\star, D_n)$, not the information-gain objective $I(Z_x; J \mid D_n)$ alone: an observation maximally informative about the latent can be worthless, or by \Cref{cor:insupport} strictly harmful, for the prediction target if its information is redundant given $Y_\star$.
\end{corollary}
 
\begin{remark}[Acquisition can create epistemic uncertainty]
\label{rem:create}
$\EU$ is not even monotone under acquisition in general. With $J, Y_\star$ independent uniform bits and a sensor $Z = J \oplus Y_\star$, $I(J; Y_\star) = 0$ but $I(J; Y_\star \mid Z) = \log 2$: a definition of epistemic uncertainty as ``what data reduces'' must contend with data that manufactures it.
\end{remark}
 
\section{Ensemble Disagreement Tracks Optimization, Not EU}
\label{sec:ensembles}
 
Ensembles of independently trained networks are the field's deployed epistemic estimator. Dispersion across members, typically $\hat I_{\mathrm{ens}} = H(\bar q_{\mathrm{ens}}) - \tfrac1M \sum_m H(q_m)$, is reported as the epistemic uncertainty and consumed by downstream decisions from adversarial-example detection \citep{smith2018understanding} to acquisition \citep{galislam2017deep}. We prove that in two complementary training regimes, bracketing practice from either side, the dispersion is controlled by the training procedure rather than by $\EU$. When the population objective pins the predictor down in the amortized regime of meta-trained, prior-fitted, and in-context predictors, dispersion collapses to zero beneath a positive $\EU$ (\Cref{thm:collapse}). When it does not, such as in fixed-dataset interpolation, the regime nearest a deep ensemble trained once on one dataset, dispersion is propagated initialization noise at a scale set by an optimization hyperparameter (\Cref{thm:inflation}). $\EU$ appears in neither formula.

\begin{assumption}
\label{ass:ensembles}
(A1) The predictor class $\cQ$ contains the Bayes predictive $\bar p(\cdot \mid D_n)$, and the population meta-risk $L(q) = \E[-\log q(Y_\star \mid D_n)]$ has $\bar p$ as its unique minimizer in $\cQ$ up to null modifications (automatic for unrestricted $\cQ$ by strict propriety of the log score). (A2) Training with resources $T$ and seed $\xi$ returns $q_{T,\xi}$ with $\E_{D_n}\, \dTV(q_{T,\xi}, \bar p) \to 0$ in probability over $\xi$ as $T \to \infty$. (A3) The dispersion functional $\Delta: \cQ^M \to \R_{\ge 0}$ is bounded, $\dTV$-continuous in each argument, and zero on the diagonal; on finite label alphabets this covers variance-of-means and $\hat I_{\mathrm{ens}}$.
\end{assumption}
 
\begin{theorem}[Collapse]
\label{thm:collapse}
Under \Cref{ass:ensembles}, $\E\, \Delta(q_{T,\xi_1}, \dots, q_{T,\xi_M}) \to 0$ as $T \to \infty$ for every ensemble size $M$, while $\EU(n)$ is a property of the problem alone and unchanged; in the two-region construction, $\EU(n) = \log 2 - h(\varepsilon) > 0$. The ensemble functional is therefore an inconsistent estimator of $\EU$ whenever $\EU_\infty^{\cA_\rho} > 0$.
\end{theorem}
 
\begin{proof}
By (A2) and independence of the seeds, the tuple $(q_{T,\xi_1}, \dots, q_{T,\xi_M})$ converges jointly in probability to $(\bar p, \dots, \bar p)$; by (A3), $\Delta$ converges to $\Delta(\bar p, \dots, \bar p) = 0$ in probability, and boundedness upgrades convergence to expectation. $\EU$ does not reference the algorithm.
\end{proof}
 
Seed randomness samples the algorithm's posterior, not the posterior over $J$. A faithful posterior-sampling ensemble in the two-region construction would have members predicting $\Bern(\varepsilon)$ or $\Bern(1{-}\varepsilon)$ with equal probability and would report $\hat I_{\mathrm{ens}} \approx \EU$ correctly. Consistent training collapses exactly the spread that was carrying the signal, because the unique log-loss minimizer is the hedged $\bar p$ and every member converges to it. The Bayesian reading of ensembles is invalidated by the thing that makes training work. 

\begin{theorem}[Inflation, closed form]
\label{thm:inflation}
Consider linear prediction $f_w(x) = x^\top w$ with $d > n$, noiseless data $y = Xw^\ast$, and gradient flow on $\tfrac12 \lVert Xw - y \rVert^2$ from $w_0 \sim \mathcal{N}(0, \alpha^2 I_d)$. The flow remains in $w_0 + \mathrm{row}(X)$ and converges to $w_\infty = P_N w_0 + X^\top(XX^\top)^{-1} y$ with $P_N = I - X^\top(XX^\top)^{-1}X$, so the ensemble predictive variance at any query $x_\star$ is $\alpha^2 \lVert P_N x_\star \rVert^2$. Against a point-mass prior on $w^\ast$ (true $\EU \equiv 0$) this is arbitrarily large in $\alpha$. Against a $\mathcal{N}(0, \tau^2 I)$ prior, the exact posterior predictive variance at $x_\star$ is $\tau^2 \lVert P_N x_\star \rVert^2$, so the ensemble misestimates the epistemic term by the factor $(\alpha/\tau)^2$, unbounded in both directions and controlled by an optimization hyperparameter.
\end{theorem}
 
\begin{remark}[Complementary regimes]
\label{rem:regimes}
\Cref{thm:collapse} assumes a unique population minimizer; \Cref{thm:inflation} is the degenerate case of an interpolation manifold of minimizers. Together when the objective pins the predictor down, disagreement collapses below $\EU$. When it does not, disagreement equals propagated initialization noise at a scale set by $\alpha$. In neither regime does $\EU$ appear in the formula for disagreement, and practical ensembles such as finite $T$, tuned $\alpha$, and early stopping interpolate between the regimes. \Cref{thm:inflation} is the sample-then-optimize mechanism of \citet{osband2018randomized} used contrapositively. That literature shows ensembles become posterior samples if initialization is drawn from the prior and the procedure is anchored, and the theorem is what remains when, as in practice, it is not.
\end{remark}
 
\section{A Falsification Test}
\label{sec:test}
 
\Cref{cor:insupport} gives mechanism-irreducibility a testable signature. Under exact unidentifiability, $\EU(n)$ is monotone non-decreasing, so a statistically significant decrease falsifies it. In the synthetic-prior regime, like the regime of prior-fitted networks and neural processes, where the components are known at training time, the per-context epistemic term $I_i = I(J; Y_\star \mid D_n = d_i)$ is computable exactly and bounded by $\log M$ for an $M$-component prior, so the test comes with finite-sample type-I control.
 
\begin{proposition}[Level-$\alpha$ falsification test]
\label{prop:test}
Let $H_0$ assert exact unidentifiability at $x_\star$. Estimate $\widehat{\EU}(n_1)$ and $\widehat{\EU}(n_2)$, $n_1 < n_2$, by Monte Carlo over $m_1, m_2$ i.i.d.\ contexts, and set $t_i = \log M \sqrt{\log(2/\alpha)/(2 m_i)}$. The test that rejects $H_0$ when $\widehat{\EU}(n_2) - \widehat{\EU}(n_1) < -(t_1 + t_2)$ has $\Prob_{H_0}(\mathrm{reject}) \le \alpha$.
\end{proposition}
 
\begin{proof}
Under $H_0$, $\EU(n_2) \ge \EU(n_1)$ by \Cref{cor:insupport}, so rejection requires an estimation error exceeding $t_1$ or $t_2$; each $I_i \in [0, \log M]$, so Hoeffding bounds each event's probability by $\alpha/2$, and the union bound completes.
\end{proof}
 
The scope is one-sided, and we state the limits. Observed flatness does not certify mechanism-irreducibility, since a sample-reducible component with a slow rate is flat at finite $n$; certification requires the prior, which the synthetic regime supplies and the black-box regime does not. And in the black-box regime the test needs a trustworthy estimator of $\EU$, which is what \Cref{thm:collapse} says ensembles are not. whether weaker black-box evidence of mechanism-irreducibility is attainable is an open question.
 
\section{Experiments}
\label{sec:experiments}

\begin{figure*}[htbp]
\centering
\includegraphics[width=\linewidth]{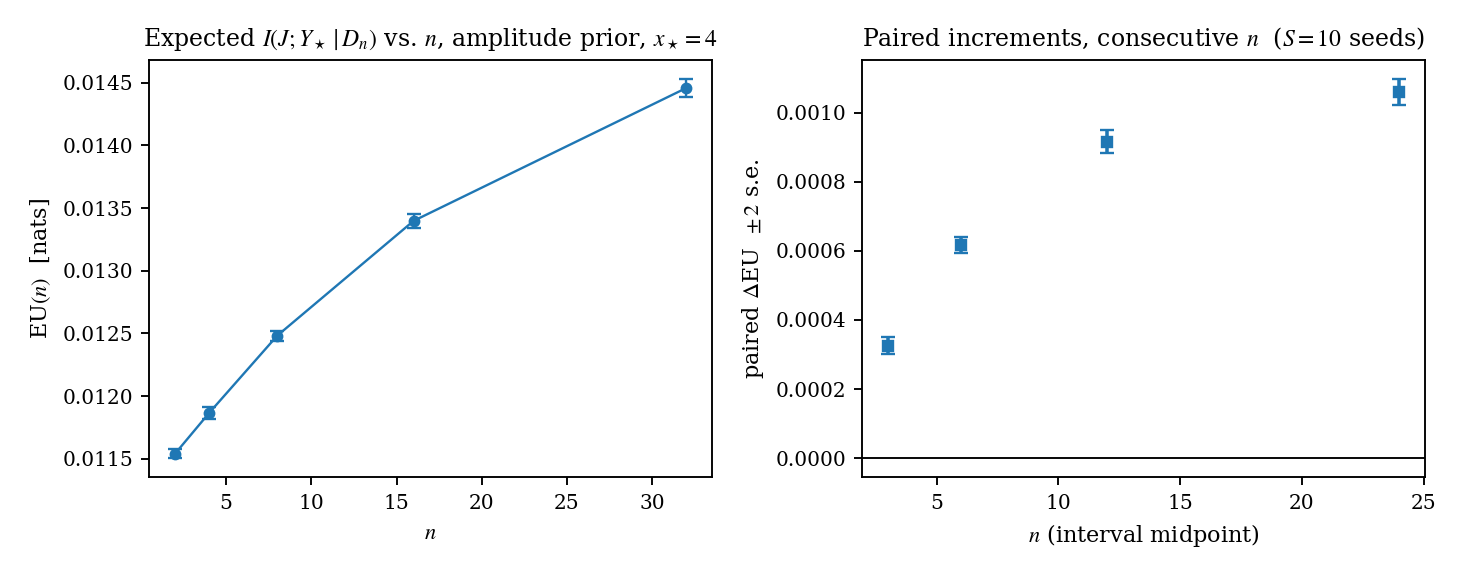}
\caption{The epistemic term strictly increases with in-distribution sample size under exact unidentifiability (\Cref{cor:insupport}). Left: $\E\, I(J; Y_\star \mid D_n)$ versus $n$ on the GP amplitude construction at $x_\star = 4$, mean $\pm 2$ s.e.\ over $10$ seeds. Right: paired increments between consecutive $n$; every increment is positive ($p \le 3.7 \times 10^{-10}$).}
\label{fig:increase}
\end{figure*}

\textbf{Validation of \cref{eq:schur}.} On the GP amplitude construction (squared-exponential correlation with $\ell = 1$, observation noise $\sigma^2 = 0.04$, amplitude slopes $\{0, 0.25\}$ activating outside the support $[-3, 3]$), the closed-form slope and variance of \cref{eq:schur} match brute-force Gaussian conditioning over random instances to maximum error $7.5 \times 10^{-16}$ and $4.8 \times 10^{-15}$. At $x_\star = 4$ with nested contexts (paired increments, $2000$ contexts per seed, $10$ seeds), $\EU(n)$ rises from $0.01154 \pm 0.00005$ at $n = 2$ to $0.01446 \pm 0.00011$ at $n = 32$, and every paired increment is positive with across-seed one-sided $p$-values between $3.7 \times 10^{-10}$ and $3.6 \times 10^{-13}$ (\Cref{fig:increase}). The increase decays with query distance as the $O(\delta^2)$ analysis predicts. At $x_\star = 8$ the predicted increment is of order $e^{-25}$, and the measured curve is correspondingly constant to five digits ($0.1155$) across all $n$. Apparent flatness of epistemic-term-versus-$n$ curves at far queries should therefore be read as monotone increase beneath measurement precision.

\begin{figure}[htbp]
\centering
\includegraphics[width=\linewidth]{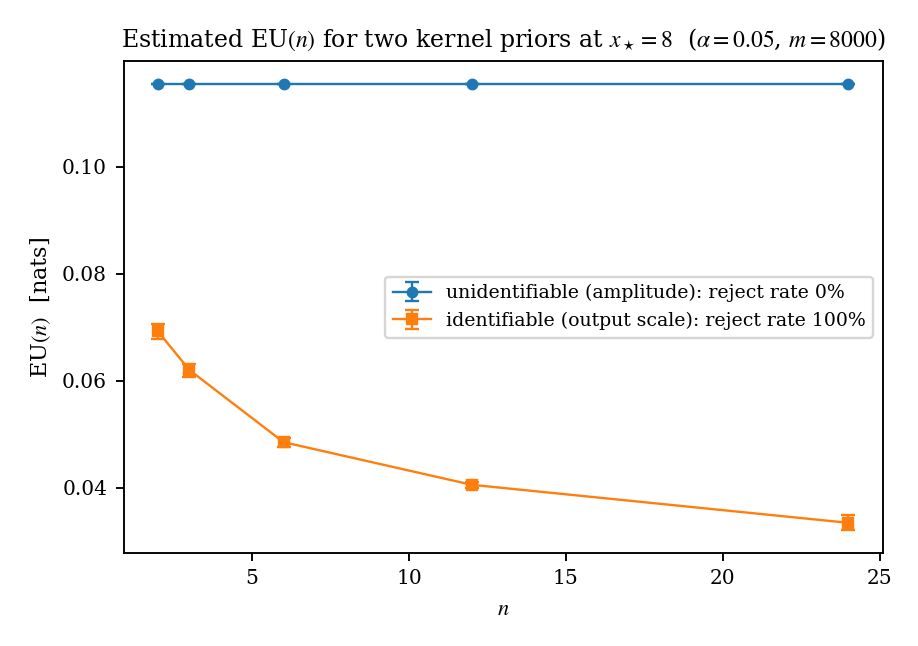}
\caption{$\EU(n)$ profiles for the two controls at $x_\star = 8$. The unidentifiable amplitude prior is flat (strictly increasing below numerical visibility) and the test fails to reject in $10/10$ runs. The identifiable output-scale prior decays and the test rejects in $10/10$ runs at the proven Hoeffding threshold.}
\label{fig:test}
\end{figure}

\textbf{Falsification test on both controls.} With $n_1 = 2$, $n_2 = 24$, $m = 8000$ contexts per arm, and $\alpha = 0.05$ (threshold $-0.02105$), ten independent tests on the unidentifiable amplitude prior all fail to reject ($T = +0.00000 \pm 0.00000$), and ten tests on an identifiable output-scale prior (stationary kernels with output variances $1$ and $4$) all reject ($T = -0.0359 \pm 0.0005$); see \Cref{fig:test}. The statistic on the positive control is identically zero to five digits across $160{,}000$ contexts because at far queries the per-context epistemic term concentrates and the deep-tail floor barely fluctuates over contexts, so the test passes not weakly but exactly. A lengthscale pair at a far query has $\EU \equiv 0$ at every $n$, because both unit-variance components revert to the same prior predictive there. Kernel identifiability from data and component disagreement at the query are independent requirements, and the test needs both. At moderate queries the lengthscale pair's $\EU(n)$ is a bump (rising to $n \approx 12$, then decaying), which is an instance of \Cref{rem:create} where in-distribution data transiently creating epistemic uncertainty before identification wins.

\begin{figure}[htbp]
\centering
\includegraphics[width=\linewidth]{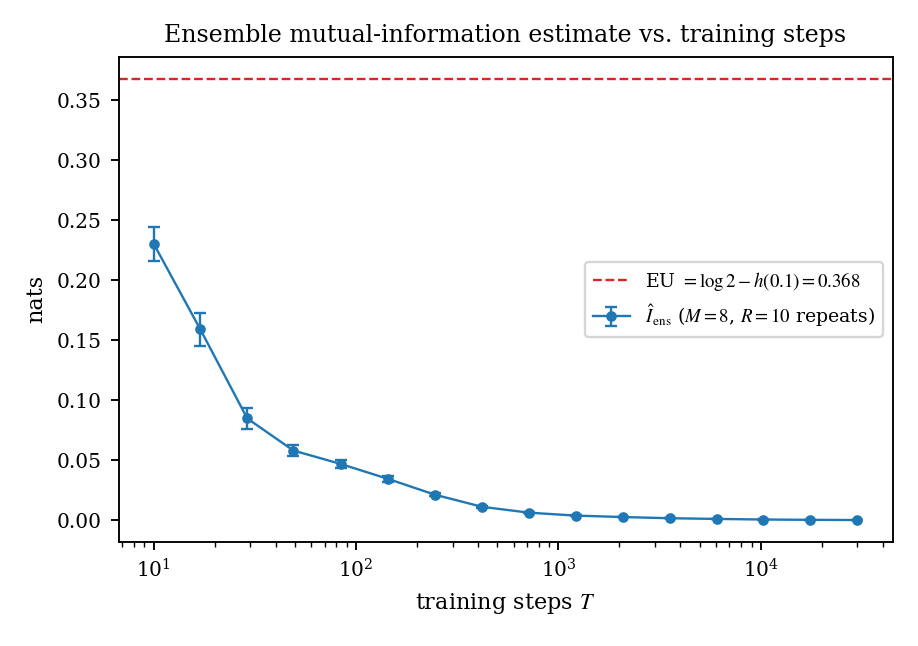}
\caption{The ensemble mutual-information estimate collapses by three orders of magnitude under training while its estimand $\EU = 0.3681$ (dashed) is constant. Mean $\pm 2$ s.e.\ over $10$ independent ensemble repeats.}
\label{fig:collapse}
\end{figure}

\textbf{Collapse.} On the two-region construction with $\varepsilon = 0.1$ ($\EU = \log 2 - h(0.1) = 0.3681$ nats, flat in $T$ by construction), $8$-member ensembles of a small MLP trained by log loss show $\hat I_{\mathrm{ens}}$ falling monotonically from $0.231 \pm 0.022$ at $T = 10$ to $0.00035 \pm 0.00003$ at $T = 3 \times 10^4$, over $10$ independent ensemble repeats (\Cref{fig:collapse}). The estimator is wrong in both phases. At initialization it reflects init dispersion (the \Cref{thm:inflation} knob), and trained it reflects nothing. At no point does it pass through the truth for a reason connected to the truth.

\begin{figure}[htbp]
\centering
\includegraphics[width=\linewidth]{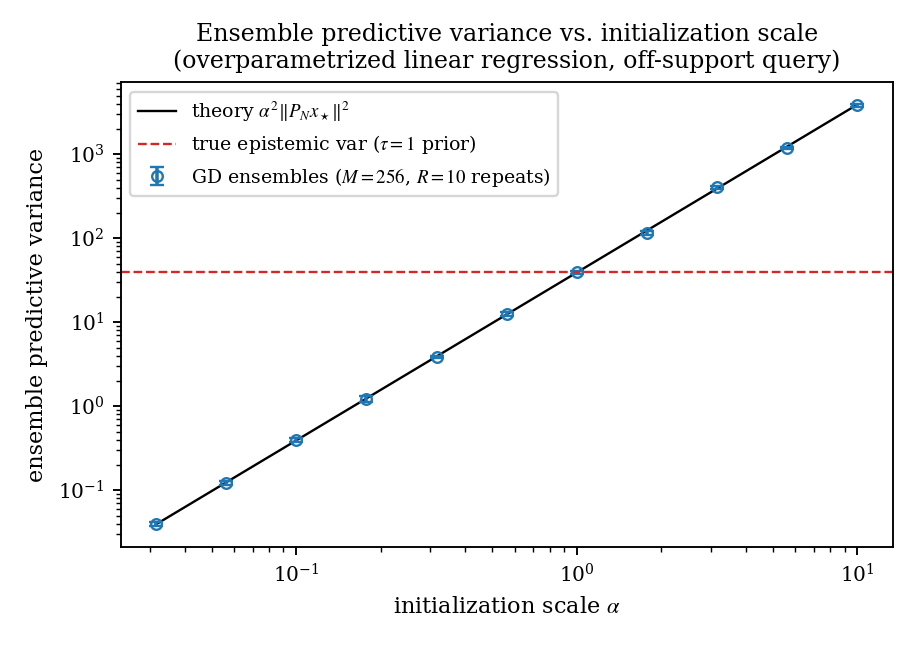}
\caption{Ensemble predictive variance at an off-support query equals $\alpha^2 \lVert P_N x_\star \rVert^2$ (solid) and crosses the true epistemic variance (dashed) only where the initialization scale happens to equal the prior scale.}
\label{fig:inflation}
\end{figure}

\textbf{Inflation.} In the overparametrized linear regime ($d = 60$, $n = 20$), gradient-descent ensembles of $256$ members match the closed form $\alpha^2 \lVert P_N x_\star \rVert^2$ across two and a half decades of initialization scale with median relative error $1.9\%$ ($10$ repeats), crossing the true Bayes epistemic variance only at the accidental calibration $\alpha = \tau$ (\Cref{fig:inflation}).
  
\section{Discussion}
\label{sec:discussion}

The results separate into a negative half and a constructive half. The negative half is \Cref{thm:partition,thm:inconsistency,thm:collapse,thm:inflation}. The taxonomy's two commitments contradict each other, and the deployed estimator of the epistemic term tracks optimization geometry rather than information in both of its regimes. The constructive half is the trichotomy with its acquisition-relative vocabulary, the value identity of \Cref{thm:value} with its corrected acquisition objective, and the falsification test of \Cref{prop:test}. We close with the operational consequences and the limits.

We posit that the trichotomy changes what an uncertainty number is for. A system that reports one epistemic number conflates two different recommendations to collect more data and collect different data. The constructions show the conflation is not hypothetical because mechanism-reducible uncertainty is exactly the regime of extrapolation under latent unidentifiability, which is where uncertainty estimates are consumed by Bayesian optimization, active learning, and safety fallbacks. Every acquisition loop that routes high-epistemic inputs to more deployment-distribution sampling spends budget that \Cref{cor:insupport} proves is worth at most zero, and the corrected objective of \Cref{cor:acquisition} prices the alternative mechanisms. Symmetrically, no recalibration removes mechanism-reducible uncertainty. A predictor at the floor passes every calibration diagnostic, because the marginal predictive is the true conditional law of the observable. The failure is therefore invisible to the validation toolkit and untouchable by post-hoc correction.

A reasonable objection is that no one holds (D) as a literal definition. Asked directly, most practitioners would say they mean something weaker, such as something like epistemic uncertainty is reducible in principle, by some data, and \Cref{rem:retreat} already shows this retreat relativizes reducibility to an acquisition class, which is the trichotomy. But the professed belief is not where the commitment lives, which instead lives in the artifacts. An acquisition loop that routes high-epistemic inputs to more deployment-distribution sampling encodes (D) over $\cA_\rho$ whether or not its authors would defend (D) in print. 

A validation protocol that certifies an estimator by exhibiting epistemic uncertainty decaying with dataset size encodes it. An estimator that hardcodes the decay into its functional form \citep{amini2020deep} encodes it most literally of all. The gap between the stated belief and the deployed one is the phenomenon. The literal reading of (D) is rarely defended and routinely operationalized, and our constructions bear on the operationalization, not on the talk-slide intuition. This is also why the failure resists the usual checks since a predictor at the floor passes every calibration diagnostic, so the toolkit that would catch a miscalibrated estimator is silent on one that is well-calibrated and measuring the wrong quantity. The commitment is invisible in two senses at once, unprofessed in the literature and undetectable by the validation suite, while remaining active in the estimators that use the number.

\bibliographystyle{unsrtnat}
\bibliography{bibliography}

@inproceedings{amini2020deep,
author = {Amini, Alexander and Schwarting, Wilko and Soleimany, Ava and Rus, Daniela},
title = {Deep evidential regression},
year = {2020},
isbn = {9781713829546},
publisher = {Curran Associates Inc.},
address = {Red Hook, NY, USA},
booktitle = {Proceedings of the 34th International Conference on Neural Information Processing Systems},
articleno = {1251},
numpages = {11},
url = {http://arxiv.org/abs/1910.02600},
location = {Vancouver, BC, Canada},
series = {NIPS '20}
}

@misc{baan2023uncertainty,
	title = {Uncertainty in {Natural} {Language} {Generation}: {From} {Theory} to {Applications}},
	shorttitle = {Uncertainty in {Natural} {Language} {Generation}},
	url = {http://arxiv.org/abs/2307.15703},
	doi = {10.48550/arXiv.2307.15703},
	urldate = {2026-06-09},
	publisher = {arXiv},
	author = {Baan, Joris and Daheim, Nico and Ilia, Evgenia and Ulmer, Dennis and Li, Haau-Sing and Fernández, Raquel and Plank, Barbara and Sennrich, Rico and Zerva, Chrysoula and Aziz, Wilker},
	month = jul,
	year = {2023},
	note = {arXiv:2307.15703},
}

@inproceedings{bengs2022pitfalls,
title={Pitfalls of Epistemic Uncertainty Quantification through Loss Minimisation},
author={Viktor Bengs and Eyke H{\"u}llermeier and Willem Waegeman},
booktitle={Advances in Neural Information Processing Systems},
editor={Alice H. Oh and Alekh Agarwal and Danielle Belgrave and Kyunghyun Cho},
year={2022},
url={https://openreview.net/forum?id=epjxT_ARZW5}
}

@InProceedings{bickfordsmith2023prediction,
  title = 	 {Prediction-Oriented Bayesian Active Learning},
  author =       {Bickford Smith, Freddie and Kirsch, Andreas and Farquhar, Sebastian and Gal, Yarin and Foster, Adam and Rainforth, Tom},
  booktitle = 	 {Proceedings of The 26th International Conference on Artificial Intelligence and Statistics},
  pages = 	 {7331--7348},
  year = 	 {2023},
  editor = 	 {Ruiz, Francisco and Dy, Jennifer and van de Meent, Jan-Willem},
  volume = 	 {206},
  series = 	 {Proceedings of Machine Learning Research},
  month = 	 {25--27 Apr},
  publisher =    {PMLR},
  url = 	 {https://proceedings.mlr.press/v206/bickfordsmith23a.html},
}

@InProceedings{depeweg2018decomposition,
  title = 	 {Decomposition of Uncertainty in {B}ayesian Deep Learning for Efficient and Risk-sensitive Learning},
  author =       {Depeweg, Stefan and Hernandez-Lobato, Jose-Miguel and Doshi-Velez, Finale and Udluft, Steffen},
  booktitle = 	 {Proceedings of the 35th International Conference on Machine Learning},
  pages = 	 {1184--1193},
  year = 	 {2018},
  editor = 	 {Dy, Jennifer and Krause, Andreas},
  volume = 	 {80},
  series = 	 {Proceedings of Machine Learning Research},
  month = 	 {10--15 Jul},
  publisher =    {PMLR},
  url = 	 {https://proceedings.mlr.press/v80/depeweg18a.html},
}

@article{derkiureghian2009aleatory,
	title = {Aleatory or epistemic? {Does} it matter?},
	volume = {31},
	issn = {01674730},
	shorttitle = {Aleatory or epistemic?},
	url = {https://linkinghub.elsevier.com/retrieve/pii/S0167473008000556},
	doi = {10.1016/j.strusafe.2008.06.020},
	language = {en},
	number = {2},
	urldate = {2026-06-09},
	journal = {Structural Safety},
	author = {Kiureghian, Armen Der and Ditlevsen, Ove},
	month = mar,
	year = {2009},
	pages = {105--112},
}

@article{foster1998asymptotic,
	title = {Asymptotic calibration},
	volume = {85},
	issn = {0006-3444, 1464-3510},
	url = {https://academic.oup.com/biomet/article-lookup/doi/10.1093/biomet/85.2.379},
	doi = {10.1093/biomet/85.2.379},
	language = {en},
	number = {2},
	urldate = {2026-06-09},
	journal = {Biometrika},
	author = {Foster, D.},
	month = jun,
	year = {1998},
	pages = {379--390},
}

@InProceedings{galislam2017deep,
  title = 	 {Deep {B}ayesian Active Learning with Image Data},
  author =       {Yarin Gal and Riashat Islam and Zoubin Ghahramani},
  booktitle = 	 {Proceedings of the 34th International Conference on Machine Learning},
  pages = 	 {1183--1192},
  year = 	 {2017},
  editor = 	 {Precup, Doina and Teh, Yee Whye},
  volume = 	 {70},
  series = 	 {Proceedings of Machine Learning Research},
  month = 	 {06--11 Aug},
  publisher =    {PMLR},
  url = 	 {https://proceedings.mlr.press/v70/gal17a.html},
}

@misc{gruber2023sources,
	title = {Sources of {Uncertainty} in {Supervised} {Machine} {Learning} -- {A} {Statisticians}' {View}},
	url = {http://arxiv.org/abs/2305.16703},
	doi = {10.48550/arXiv.2305.16703},
	urldate = {2026-06-09},
	publisher = {arXiv},
	author = {Gruber, Cornelia and Schenk, Patrick Oliver and Schierholz, Malte and Kreuter, Frauke and Kauermann, Göran},
	month = jan,
	year = {2025},
	note = {arXiv:2305.16703},
}

@misc{henning2022ood,
	title = {On out-of-distribution detection with {Bayesian} neural networks},
	url = {http://arxiv.org/abs/2110.06020},
	doi = {10.48550/arXiv.2110.06020},
	urldate = {2026-06-09},
	publisher = {arXiv},
	author = {D'Angelo, Francesco and Henning, Christian},
	month = feb,
	year = {2022},
	note = {arXiv:2110.06020},
}

@misc{houlsby2011bald,
	title = {Bayesian {Active} {Learning} for {Classification} and {Preference} {Learning}},
	url = {http://arxiv.org/abs/1112.5745},
	doi = {10.48550/arXiv.1112.5745},
	urldate = {2026-06-09},
	publisher = {arXiv},
	author = {Houlsby, Neil and Huszár, Ferenc and Ghahramani, Zoubin and Lengyel, Máté},
	month = dec,
	year = {2011},
	note = {arXiv:1112.5745},
}

@article{hullermeier2021aleatoric,
  author    = {H{\"u}llermeier, Eyke and Waegeman, Willem},
  title     = {Aleatoric and Epistemic Uncertainty in Machine Learning: An Introduction to Concepts and Methods},
  journal   = {Machine Learning},
  volume    = {110},
  number    = {3},
  pages     = {457--506},
  year      = {2021},
  doi       = {10.1007/s10994-021-05946-3},
  url       = {https://doi.org/10.1007/s10994-021-05946-3}
}

@inproceedings{kendall2017uncertainties,
author = {Kendall, Alex and Gal, Yarin},
title = {What uncertainties do we need in Bayesian deep learning for computer vision?},
year = {2017},
isbn = {9781510860964},
publisher = {Curran Associates Inc.},
address = {Red Hook, NY, USA},
booktitle = {Proceedings of the 31st International Conference on Neural Information Processing Systems},
url = {http://arxiv.org/abs/1703.04977},
pages = {5580–5590},
numpages = {11},
location = {Long Beach, California, USA},
series = {NIPS'17}
}

@inbook{kirsch2019batchbald,
author = {Kirsch, Andreas and Amersfoort, Joost van and Gal, Yarin},
title = {BatchBALD: efficient and diverse batch acquisition for deep Bayesian active learning},
year = {2019},
publisher = {Curran Associates Inc.},
address = {Red Hook, NY, USA},
booktitle = {Proceedings of the 33rd International Conference on Neural Information Processing Systems},
articleno = {631},
numpages = {12}
}

@inproceedings{malinin2018predictive,
author = {Malinin, Andrey and Gales, Mark},
title = {Predictive uncertainty estimation via prior networks},
year = {2018},
url = {http://arxiv.org/abs/1802.10501},
publisher = {Curran Associates Inc.},
address = {Red Hook, NY, USA},
booktitle = {Proceedings of the 32nd International Conference on Neural Information Processing Systems},
pages = {7047–7058},
numpages = {12},
location = {Montr\'{e}al, Canada},
series = {NIPS'18}
}

@inproceedings{kirchhof2025reexamining,
  author = {Kirchhof, Michael and Kasneci, Gjergji and Kasneci, Enkelejda},
  title = {Reexamining the Aleatoric and Epistemic Uncertainty Dichotomy},
  booktitle = {ICLR Blogposts 2025},
  year = {2025},
  date = {April 28, 2025},
  note = {https://iclr-blogposts.github.io/2025/blog/reexamining-the-aleatoric-and-epistemic-uncertainty-dichotomy/},
  url  = {https://iclr-blogposts.github.io/2025/blog/reexamining-the-aleatoric-and-epistemic-uncertainty-dichotomy/}
}

@inproceedings{osband2018randomized,
author = {Osband, Ian and Aslanides, John and Cassirer, Albin},
title = {Randomized prior functions for deep reinforcement learning},
year = {2018},
url = {http://arxiv.org/abs/1806.03335},
publisher = {Curran Associates Inc.},
address = {Red Hook, NY, USA},
booktitle = {Proceedings of the 32nd International Conference on Neural Information Processing Systems},
pages = {8626–8638},
numpages = {13},
location = {Montr\'{e}al, Canada},
series = {NIPS'18}
}

@misc{ruegamer2026epistemic,
	title = {On the {Epistemic} {Uncertainty} of {Overparametrized} {Neural} {Networks}},
	url = {http://arxiv.org/abs/2605.25234},
	doi = {10.48550/arXiv.2605.25234},
	urldate = {2026-06-09},
	publisher = {arXiv},
	author = {Rügamer, David},
	month = may,
	year = {2026},
	note = {arXiv:2605.25234},
}

@inproceedings{smith2018understanding,
  author    = {Smith, Lewis and Gal, Yarin},
  title     = {Understanding Measures of Uncertainty for Adversarial Example Detection},
  booktitle = {Proceedings of the 34th Conference on Uncertainty in Artificial Intelligence (UAI 2018)},
  pages      = {560--569},
  year       = {2018},
  url        = {https://auai.org/uai2018/proceedings/papers/207.pdf}
}

@article{walmsley2020galaxy,
  author  = {Walmsley, Mike and Smith, Lewis and Lintott, Chris and Gal, Yarin and Bamford, Steven and Dickinson, Hugh and Fortson, Lucy and Kruk, Sandor and Masters, Karen and Scarlata, Claudia and Simmons, Brooke and Smethurst, Rebecca and Wright, Darryl},
  title   = {Galaxy Zoo: Probabilistic Morphology through Bayesian CNNs and Active Learning},
  journal = {Monthly Notices of the Royal Astronomical Society},
  volume  = {491},
  number  = {2},
  pages   = {1554--1574},
  year    = {2020},
  month   = jan,
  doi     = {10.1093/mnras/stz2816},
  url     = {https://doi.org/10.1093/mnras/stz2816}
}

@inproceedings{chen2026querylevel,
title={Query-Level Uncertainty in Large Language Models},
author={Lihu Chen and Gerard de Melo and Fabian M. Suchanek and Ga{\"e}l Varoquaux},
booktitle={The Fourteenth International Conference on Learning Representations},
year={2026},
url={https://openreview.net/forum?id=11QZITAMUO}
}

@InProceedings{wimmer2023quantifying,
  title = 	 {Quantifying aleatoric and epistemic uncertainty in machine learning: Are conditional entropy and mutual information appropriate measures?},
  author =       {Wimmer, Lisa and Sale, Yusuf and Hofman, Paul and Bischl, Bernd and H\"ullermeier, Eyke},
  booktitle = 	 {Proceedings of the Thirty-Ninth Conference on Uncertainty in Artificial Intelligence},
  pages = 	 {2282--2292},
  year = 	 {2023},
  editor = 	 {Evans, Robin J. and Shpitser, Ilya},
  volume = 	 {216},
  series = 	 {Proceedings of Machine Learning Research},
  month = 	 {31 Jul--04 Aug},
  publisher =    {PMLR},
  url = 	 {https://proceedings.mlr.press/v216/wimmer23a.html},
}

@article{xu2022minimum,
	title = {Minimum {Excess} {Risk} in {Bayesian} {Learning}},
	volume = {68},
	copyright = {https://creativecommons.org/licenses/by/4.0/legalcode},
	issn = {0018-9448, 1557-9654},
	url = {https://ieeexplore.ieee.org/document/9780255/},
	doi = {10.1109/TIT.2022.3176056},
	number = {12},
	urldate = {2026-06-09},
	journal = {IEEE Transactions on Information Theory},
	author = {Xu, Aolin and Raginsky, Maxim},
	month = dec,
	year = {2022},
	pages = {7935--7955},
}

\clearpage
\appendix
\thispagestyle{empty}

\onecolumn

\aistatstitle{Appendix}
\crefalias{section}{appendix}
\crefalias{subsection}{appendix}

\section{Proof of \Cref{cor:insupport} (Schur computation)}
\label{app:schur}
 
The inequality $\Delta\EU(Z) \le 0$ and the monotonicity of $\EU(n)$ are immediate from \Cref{thm:value}. Under exact unidentifiability the conditional law of an in-support $Z$ given $(D_n, J)$ is $J$-free, so $I(J; Z \mid D_n) = 0$ and $\Delta\EU(Z) = -\E\, I(J; Z \mid Y_\star, D_n) \le 0$; iterating over $n$ gives $\EU(n+1) \ge \EU(n)$. It remains to prove strictness on the amplitude construction, i.e.\ that $I(J; Y_z \mid Y_\star, D_n) > 0$ pointwise on a full-measure event, by showing the two conditional laws of $Y_z$ given $(D_n, Y_\star)$ differ.
 
\textit{Joint covariance.} Fix the context inputs $X = (x_1, \dots, x_n) \subset \mathrm{supp}\,\rho$, an in-support probe location $x_z$, and the query $x_\star \notin \mathrm{supp}\,\rho$. Under component $j$ the vector $(y, Y_z, Y_\star)$ is zero-mean Gaussian. Because the amplitudes satisfy $a_j \equiv 1$ on $\mathrm{supp}\,\rho$, all in-support blocks are $j$-free: $\mathrm{Var}(y) = A := K + \sigma^2 I$ with $K_{il} = c(x_i, x_l)$, $\mathrm{Cov}(y, Y_z) = c_z$ with $(c_z)_i = c(x_i, x_z)$, and $\mathrm{Var}(Y_z) = c_{zz} + \sigma^2$. The query-involving blocks scale with the query amplitude $a_j := a_j(x_\star)$: $\mathrm{Cov}(y, Y_\star) = a_j b$ with $b_i = c(x_i, x_\star)$, $\mathrm{Cov}(Y_z, Y_\star) = a_j c_{z\star}$ with $c_{z\star} = c(x_z, x_\star)$, and $\mathrm{Var}(Y_\star) = a_j^2 c_{\star\star} + \sigma^2$.
 
\textit{Conditioning.} Write $u = (y, Y_\star)$ with covariance $\Sigma_j = \begin{psmallmatrix} A & a_j b \\ a_j b^\top & a_j^2 c_{\star\star} + \sigma^2 \end{psmallmatrix}$ and cross-covariance $\mathrm{Cov}(Y_z, u) = (c_z^\top,\; a_j c_{z\star})$. The Schur complement of $A$ in $\Sigma_j$ is $s_j = a_j^2 c_{\star\star} + \sigma^2 - a_j^2\, b^\top A^{-1} b = \sigma^2 + a_j^2 \gamma$ with $\gamma := c_{\star\star} - b^\top A^{-1} b > 0$, and the block inverse is
\begin{equation}
\Sigma_j^{-1} = \begin{pmatrix} A^{-1} + \dfrac{a_j^2}{s_j} A^{-1} b\, b^\top A^{-1} & -\dfrac{a_j}{s_j} A^{-1} b \\[1ex] -\dfrac{a_j}{s_j}\, b^\top A^{-1} & \dfrac{1}{s_j} \end{pmatrix}.
\end{equation}
The conditional law of $Y_z$ given $u$ under component $j$ is Gaussian with mean $\beta_j^\top u$ where $\beta_j^\top = \mathrm{Cov}(Y_z, u)\, \Sigma_j^{-1}$. Multiplying out, and writing $\delta := c_{z\star} - c_z^\top A^{-1} b$ for the screened probe--query cross-covariance, the coefficient on $Y_\star$ is $g_j = a_j \delta / s_j$, the coefficient on $y$ is $c_z^\top A^{-1} - (a_j^2 \delta / s_j)\, b^\top A^{-1}$, and the conditional variance is
\begin{equation}
v_j = (c_{zz} + \sigma^2) - \mathrm{Cov}(Y_z, u)\, \Sigma_j^{-1}\, \mathrm{Cov}(u, Y_z) = r - \frac{a_j^2 \delta^2}{\sigma^2 + a_j^2 \gamma}, \qquad r := c_{zz} + \sigma^2 - c_z^\top A^{-1} c_z,
\end{equation}
where the middle equality uses $\mathrm{Cov}(Y_z, u)\, \Sigma_j^{-1}\, \mathrm{Cov}(u, Y_z) = c_z^\top A^{-1} c_z + a_j^2 \delta^2 / s_j$, the cross terms collecting into $(c_{z\star} - c_z^\top A^{-1} b)^2$. This establishes \cref{eq:schur}. The $n = 0$ case is recovered by taking $A$ empty: $\delta \to c_{z\star}$, $\gamma \to c_{\star\star}$, $r \to c_{zz} + \sigma^2$.
 
\textit{Degeneracy analysis.} Suppose $\delta \neq 0$ and ask when the two conditional laws coincide. Slope equality $g_0 = g_1$ reads $a_0(\sigma^2 + a_1^2\gamma) = a_1(\sigma^2 + a_0^2\gamma)$, i.e.\ $\sigma^2(a_0 - a_1) = a_0 a_1 \gamma (a_0 - a_1)$, i.e.\ $a_0 a_1 \gamma = \sigma^2$ since $a_0 \neq a_1$: a measure-zero manifold in the parameters. But variance equality $v_0 = v_1$ reads $a_0^2(\sigma^2 + a_1^2\gamma) = a_1^2(\sigma^2 + a_0^2\gamma)$, i.e.\ $\sigma^2(a_0^2 - a_1^2) = 0$, which forces $a_0 = a_1$ (amplitudes are nonnegative) and contradicts the construction, with no exceptional manifold. The variance channel alone therefore carries strictness. Whenever $\delta \neq 0$, the two conditional laws are Gaussians of different variance, the pointwise KL divergence is positive, and $I(J; Y_z \mid Y_\star, D_n = (X, y)) > 0$ for every realization of $y$.
 
\textit{The event $\{\delta(X) \neq 0\}$ has full measure.} For a strictly positive-definite, real-analytic correlation (e.g.\ squared-exponential), $\delta(X) = c_{z\star} - c_z^\top (K(X) + \sigma^2 I)^{-1} b(X)$ is a real-analytic function of the configuration $(X, x_z)$, and it is not identically zero (the empty-context case gives $\delta = c_{z\star} > 0$). The zero set of a non-trivial real-analytic function is Lebesgue-null, so under non-atomic $\rho$ the event $\{\delta \neq 0\}$ has probability one, and taking expectations preserves strict positivity: $\E\, I(J; Y_z \mid Y_\star, D_n) > 0$.
 
\textit{Leading order.} The slope difference is $\Delta g = \delta\,(a_0/s_0 - a_1/s_1) = O(\delta)$ and contributes $O(\delta^2)$ to the KL divergence through the mean channel; the variance difference is $\Delta v = \delta^2 (a_1^2/s_1 - a_0^2/s_0) = O(\delta^2)$ and contributes $O(\delta^4)$. The increase is therefore $O(\delta^2)$ to leading order, and $\delta$ inherits the decay of $c(\cdot, x_\star)$: for a squared-exponential correlation with lengthscale $\ell$, the increment at query distance $D$ from the support is of order $e^{-D^2/\ell^2}$. The closed forms above were validated against brute-force Gaussian conditioning over random instances to maximum error $7.5 \times 10^{-16}$ (slope) and $4.8 \times 10^{-15}$ (variance). \qed

\section{Proof of \Cref{cor:overlap}}
\label{app:overlap}

\emph{Reduction to the probe count.} Condition on the covariates. Given $X_{1:n}$ and $J$, the labels are independent, and the labels at $X = 0$ have law $\Bern(\tfrac12)$ for both values of $J$; their conditional law given $J$ is therefore $J$-free, so they are independent of $(J, Y_\star)$ and of the labels at $X = 1$. Conditioning on $D_n$ is thus equivalent, for the joint law of $(J, Y_\star)$, to conditioning on the $K := \#\{i : X_i = 1\}$ labels observed at $X = 1$, which are i.i.d.\ $\Bern(q_J)$ probes in the sense of \Cref{thm:rates}(ii), with $K \sim \mathrm{Bin}(n, \eta)$ independent of $J$. Hence $\EU_\eta(n) = \E_K[\EU_K]$ with $\EU_k$ as in \Cref{thm:rates}.

\emph{Averaging the bounds.} The bounds of \Cref{thm:rates}(ii) hold for every $k \ge 0$; at $k = 0$ they read $c_\varepsilon/2 \le c_\varepsilon \le 1$, which holds in nats since $c_\varepsilon \le \log 2 < 1$ (and would fail in bits). Taking expectations against the binomial generating function $\E[s^K] = (1 - \eta + \eta s)^n$ at $s = \rho_B^2$ (lower) and $s = \rho_B$ (upper), and simplifying the lower exponent via $1 - \rho_B^2 = 1 - 4\varepsilon(1-\varepsilon) = (1 - 2\varepsilon)^2$, gives \cref{eq:overlap}. At $\eta = 1$ both bounds recover \Cref{thm:rates}(ii) with $k = n$ exactly; at $\eta = 0$ the lower bound is the constant $c_\varepsilon/2$, consistent with the exact flat value $c_\varepsilon$ of \Cref{thm:rates}(i).

\emph{Finite-budget floor.} Bernoulli's inequality gives $(1 - \eta(1-2\varepsilon)^2)^n \ge 1 - n\eta(1-2\varepsilon)^2 \ge \tfrac12$ whenever $n\eta(1-2\varepsilon)^2 \le \tfrac12$, and the lower bound of \cref{eq:overlap} gives $\EU_\eta(n) \ge c_\varepsilon/4$.

\emph{Scope.} $\EU_\eta(n)$ is the training curve under $n$ i.i.d.\ draws from $\rho_\eta$; the corresponding irreducible quantity $\EU_\infty^{\cA_{\rho_\eta}}$ is $0$ for every $\eta > 0$, which is the content of consequence (ii): the inconsistency of \Cref{thm:inconsistency} is exact only at $\eta = 0$, while the finite-budget statement is what survives off the boundary.

\section{Experimental configuration}
\label{app:experiments}

All experiments run with 10 seeds. Reported uncertainties are $\pm 2$ standard errors across seeds.
 
\textit{Strict increase (\Cref{fig:increase}).} GP amplitude construction of \Cref{sec:setup} with squared-exponential correlation with $\ell = 1$, observation noise $\sigma^2 = 0.04$, amplitude slopes $\{0, 0.25\}$ activating outside $r_0 = 3$, context inputs $X \sim \mathrm{Unif}(-3, 3)$, query $x_\star = 4$. Context sizes $n \in \{2, 4, 8, 16, 32\}$ are evaluated on nested prefixes of a single draw per replicate, giving paired increments; $2000$ replicates per seed, $10$ seeds. Under exact unidentifiability the label posterior equals the prior, and the per-context epistemic term is the Jensen--Shannon divergence between the two component posteriors, computed by $4001$-point grid quadrature spanning $\pm 9$ posterior standard deviations. Increment $p$-values are across-seed one-sided $t$-tests ($9$ degrees of freedom). The Schur validation draws $20$ random instances ($n = 7$, random probe locations) and compares \cref{eq:schur} against direct block conditioning of the joint covariance.
 
\textit{Falsification test (\Cref{fig:test}).} Positive control: the amplitude prior above. Negative control: stationary squared-exponential kernels ($\ell = 1$) with output variances $1$ and $4$, identifiable from in-support data and disagreeing at the query. Test parameters: $n_1 = 2$, $n_2 = 24$, query $x_\star = 8$, $m = 8000$ contexts per arm, $\alpha = 0.05$, giving threshold $-(t_1 + t_2) = -0.02105$ with $M = 2$ components; $10$ independent tests per prior. Per-context epistemic terms use the context-dependent label posterior $\pi_j(D_n) \propto \pi_j\, p(y \mid X, k_j)$, which reduces to the prior under exact unidentifiability; mixture entropies by the same quadrature. Profile curves use $m = 600$ contexts per point over $5$ seeds.
 
\textit{Collapse (\Cref{fig:collapse}).} Two-region construction with $\varepsilon = 0.1$, so $\EU = \log 2 - h(0.1) = 0.3681$ nats. Context: $12$ label-free Bernoulli bits encoded $\pm 1$. Model: a $12 \to 32 \to 1$ tanh MLP emitting a Bernoulli logit, trained by log loss on targets $y_\star \sim \Bern(q_J)$ over fresh tasks per batch (batch $256$), Adam with $\mathrm{lr}_0 = 3 \times 10^{-3}$ decayed as $\mathrm{lr}_0/(1 + t/1500)$. Ensembles of $M = 8$ seeds; $\hat I_{\mathrm{ens}}$ evaluated on $512$ fixed contexts at $16$ log-spaced checkpoints up to $T = 3 \times 10^4$; $10$ independent ensemble repeats with disjoint seed pools.
 
\textit{Inflation (\Cref{fig:inflation}).} Overparametrized linear regression with $d = 60$, $n = 20$, a fixed random instance, noiseless labels. Full-batch gradient descent ($\mathrm{lr} = 0.05$, $6000$ steps, vectorized over the ensemble) from initializations $w_0 \sim \mathcal{N}(0, \alpha^2 I)$, $M = 256$ members, $11$ values of $\alpha$ log-spaced in $[10^{-1.5}, 10]$, $10$ repeats; compared against $\alpha^2 \lVert P_N x_\star \rVert^2$ and the Bayes line $\tau^2 \lVert P_N x_\star \rVert^2$ with $\tau = 1$.

\end{document}